\documentclass{article}

     \PassOptionsToPackage{numbers, compress}{natbib}
     \usepackage[final]{neurips_2021}

\usepackage[utf8]{inputenc} % allow utf-8 input
\usepackage[T1]{fontenc}    % use 8-bit T1 fonts
\usepackage{hyperref}       % hyperlinks
\usepackage{url}            % simple URL typesetting
\usepackage{booktabs}       % professional-quality tables
\usepackage{amsfonts}       % blackboard math symbols
\usepackage{nicefrac}       % compact symbols for 1/2, etc.
\usepackage{microtype}      % microtypography
\usepackage{xcolor}         % colors

\usepackage{amsmath}
\usepackage{amssymb}
\usepackage{amsthm}
\usepackage{framed}
\usepackage{graphicx}
\usepackage{color}
\usepackage{algorithm}
\usepackage[noend]{algpseudocode}
\usepackage{url}
\usepackage{caption}
\usepackage{wrapfig}

\usepackage{mathtools,stackengine}
\usepackage[mathscr]{euscript}
\usepackage{subfigure}
\usepackage{multirow}
\usepackage{accents}

\usepackage{pifont}

\newcommand\hl[1]{``#1''}

\newcommand\n[1]{\bar{#1}}

\newtheorem{thm}{Theorem}

\newtheorem{corollary}{Corollary}

\newtheorem{definition}{Definition}

\newcommand\shrink[1]{}

\def\pr{{\it Pr}}

\def\e{{\bf e}}

\def\pa{{\bf p}}

\def\U{{\bf U}}
\def\u{{\bf u}}
\def\V{{\bf V}}

\def\X{{\bf X}}
\def\x{{\bf x}}
\def\Y{{\bf Y}}
\def\y{{\bf y}}
\def\Z{{\bf Z}}
\def\z{{\bf z}}

\def\eql(#1,#2){{#1\!\!=\!#2}}

\def\eql(#1,#2){{#1\!=\!#2}}

\def\clap#1{\hbox to 0pt{\hss#1\hss}}

\def\VE{{\sc VE}}
\def\VEC{{\sc VEC}}

\newcommand\scalemath[2]{\scalebox{#1}{\mbox{\ensuremath{\displaystyle #2}}}}

\def\eql(#1,#2){{#1\!\!=\!\!#2}}

\def\PA{{\bf P}}
\def\pa{{\bf p}}

\def\FF{{\mathcal F}}
\def\GG{{\mathcal G}}
\def\HH{{\mathcal H}}

\def\cg{{\mathbb G}}
\def\cm{{\mathcal F}}
\def\scm{{\mathcal M}}
\def\ws{{\mathcal W}}
\def\ev{{\eta}}

\def\facs(#1){\FF_{#1}}
\def\facsp(#1){\accentset{\frown}{\FF}_{#1}}
\def\facsn(#1){\accentset{\smile}{\FF}_{#1}}

\def\cls{\mathtt{cls}}
\def\sep{\mathtt{sep}}

\def\vars{\mathtt{vars}}

\def\fvars{\mathtt{fvars}}

\def\fsum{\mathtt{fsum}}

\def\mes(#1,#2){{\mathscr{M}(#1,#2)}}

\title{Causal Inference Using Tractable  Circuits}

\author{%
 Adnan Darwiche \\
Computer Science Department \\
University of California, Los Angeles \\
\texttt{darwiche@cs.ucla.edu} 
}

\begin{document}

\maketitle

\begin{abstract}
The aim of this paper is to discuss a recent result which shows that probabilistic inference in the 
presence of (unknown) causal mechanisms can be tractable for models that have traditionally been viewed as intractable. 
This result was reported recently in~\cite{DarwicheECAI20b} to facilitate model-based supervised learning 
but it can be interpreted in a causality context as follows. One can compile a non-parametric causal graph into 
an arithmetic circuit that supports inference in time linear in the circuit size. The circuit is also non-parametric so it can be used to 
estimate parameters from data and to further reason (in linear time) about the causal graph parametrized by these estimates.
Moreover, the circuit size can sometimes be bounded even when the treewidth of the causal graph is not, leading to tractable inference 
on models that have been deemed intractable previously. This has been enabled by a new technique that can exploit 
causal mechanisms computationally but without needing to know their identities (the classical setup in causal inference). 
Our goal is to provide a causality-oriented exposure to these new results and 
to speculate on how they may potentially contribute to more scalable and versatile causal inference.
\end{abstract}

\vspace{-5mm}
\section{Introduction}

Tractable arithmetic circuits have been receiving an increased attention in AI and computer
science more broadly; see~\cite{neusys22} for a recent survey. These circuits represent real-valued functions and 
are called {\em tractable} because they allow one to answer some hard queries about these functions
through linear-time, feed-forward passes on the circuit structure. 
These circuits were initially compiled from
Bayesian networks as proposed in~\cite{DarwicheJACM03,kr/Darwiche02} to facilitate probabilistic reasoning. 
They were later learned from data, starting with~\cite{LowdD08}, and even handcrafted as initially proposed in~\cite{PoonD11}.
Traditional methods for exact probabilistic inference have a complexity which is exponential 
in {\em treewidth} (a graph-theoretic parameter that measures the model's connectivity). With the introduction of compiled circuits, 
one could practically do inference on models whose treewidth can be in the hundreds; see, e.g.,~\cite{ijcai/ChaviraD05,ijar/ChaviraDJ06,ai/ChaviraD08}. 
A recent comprehensive, empirical evaluation of at least a dozen probabilistic inference algorithms showed that methods based on circuits are
at the forefront in terms of efficiency~\cite{ijcai/AgrawalPM21}; see also~\cite{uai/DilkasB21}.
What stands behind the efficacy of circuit-based methods is their ability to aggresively exploit the parametric structure
of models such as functional dependencies and context-specific independence~\cite{DBLP:journals/corr/abs-1302-3562}.
This however has limited their utility in contexts where one does not know the model parameters, which is a
common case in causal inference. A circuit compilation method was recently proposed in~\cite{DarwicheECAI20b} which can exploit a 
more abstract form of parametric structure: functional dependencies (i.e., causal mechanisms) whose identities are unknown
which is the classical setup in causal inference.
This new method was able to compile circuits for models whose treewidth is very large
without needing to know the model parameters, leading to a major advance on earlier techniques. 
Our aim in this paper is to provide
an intuitive exposure to this new algorithm and its underlying techniques, while placing it in a causality context.
Our belief is that a discussion of these new results could lead to a synthesis on how they can further
advance causal inference in terms of scalability and versatility. 
We start first with a review of some core concepts in causality and circuit-based inference 
and then follow by a discussion of the new results and their significance.

\section{Causal Models and Queries}
\label{sec:cmq}

Variables are discrete and denoted by uppercase letters (e.g., \(X\))
and their values by lowercase letters (e.g., \(x\)). Sets of variables are denoted by boldface, uppercase letters (e.g., \(\X\)) and their
instantiations by boldface, lowercase letters (e.g., \(\x\)). We will write \(x \in \x\) to mean that \(x\) is the value of variable \(X\)
in instantiation \(\x\). For a binary variable \(X\), we will use \(x\) and \(\n x\) to denote \(\eql(X,1)\) and \(\eql(X,0)\), respectively.

We next define {\em Structural Causal Models (SCMs)} following the treatment in~\cite{primer16}; see also~\cite{pearl00b}.
\begin{definition}\label{def:scm}
An SCM is a tuple \(\scm = (\U,\V,\cm,\{\pr(U)\}_{U \in \U})\) where \(\U\) and \(\V\) are disjoint sets of variables called \hl{exogenous} and \hl{endogenous,} respectively;
\(\cm\) contains exactly one function \(f_V\) for each endogenous variable \(V\); and \(\pr(U)\) are distributions over exogenous variables \(U\). 
A function \(f_V\) is called a \hl{causal mechanism} and it determines the value of variable \(V\) based on two sets of inputs, \(\U_V \subseteq \U\) 
and \(\V_V \subseteq \V\). That is, the mechanism \(f_V\) is a mapping \(\U_V,\V_V \mapsto V\).
\end{definition}
The {\em causal graph} \(\cg\) of an SCM contains variables \(\U \cup \V\) as its nodes. 
It also contains edges \(X \!\!\rightarrow\! V\) for each endogenous variable \(V\)
and each variable \(X\) that is an input to function \(f_V\). We will only deal with SCMs that 
produce acyclic causal graphs (the inputs of a function cannot depend on its output).
Moreover, we will assume that exogenous variables are independent and that only endogenous variables can be observed.
A key observation about SCMs is that once we fix the values of exogenous variables, the values of all endogenous variables are also fixed by the causal mechanisms. 

\begin{wrapfigure}[6]{r}{0.12\textwidth}
\centering
  \includegraphics[width=\linewidth]{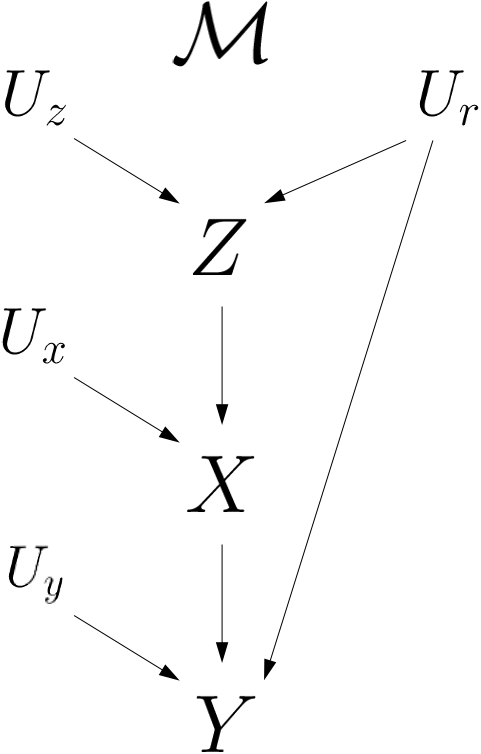}
  \caption*{\scriptsize causal graph}
\end{wrapfigure}
We will next consider an example SCM from~\cite{Bareinboim20211OP} where all variables are binary. 
The endogenous variables are \(\V = X, Y, Z\), representing 
a treatment, the outcome and the presence of hypertension, respectively. The exogenous variables are \(\U = U_r, U_x, U_y, U_z\),
representing natural resistance to disease (\(U_r\)) and sources of variation affecting endogenous variables (\(U_x, U_y, U_z\)).
The distributions of exogenous variables are
\(\pr(u_r) = 0.25\), \(\pr(u_x) = 0.9\), \(\pr(u_y) = 0.7\) and \(\pr(u_z) = 0.95\). 
The three causal mechanisms are given next and they lead to the causal graph on the right:
\begin{eqnarray}
f_X(Z,U_x)        & = &  z u_x  + \n z \n{u}_x \nonumber \\
f_Y(X,U_y,U_r) & = & x u_r + \n x u_y u_r + \n x \n{u}_y \n{u}_r \nonumber \\
f_Z(U_z,U_r)    & = & u_z u_r \nonumber
\end{eqnarray}

\begin{wrapfigure}[10]{r}{0.12\textwidth}
\centering
  \includegraphics[width=\linewidth]{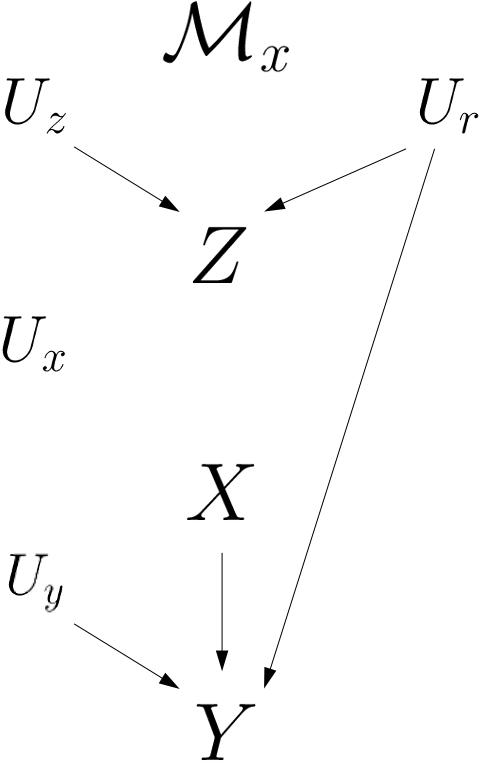}
  \caption*{\scriptsize sub-model}
\end{wrapfigure}
A key notion in causal inference is the {\em sub-model} \(\scm_\z\) of an SCM \(\scm\), where \(\z\) is an instantiation of some endogenous variables.
This is another SCM obtained from \(\scm\) by replacing the function for each variable \(Z \in \Z\) with the constant function \(f_Z = z\), where \(z \in \z\). 
Intuitively, the sub-model \(\scm_\z\) is used to reason about an {\em intervention} from outside the system modeled by \(\scm\), 
which suppresses the causal mechanisms of variables \(\Z\) and fixes the values of these variables to \(\z\).
For example, if we intervene to administer a treatment (\(x\)) in the above SCM \(\scm\), we get a sub-model \(\scm_x\) in which the mechanism for variable \(X\)
is replaced with the constant function \(f_X = x\). The causal graph of the resulting sub-model is shown on the right.

Three types of queries are normally posed on SCMs, which are referred to as {\em associational,} {\em interventional} and {\em counterfactual}
queries. They correspond to what is known as the {\em causal hierarchy}~\cite{pearl00b,pearl18,DBLP:journals/cacm/Pearl19},
where each class of queries belongs to a {\em rung}~\cite{pearl18} or {\em layer}~\cite{Bareinboim20211OP} in the hierarchy.
We next define the syntax and semantics of these queries, using a slightly different notation than is customary---this will allow us 
to provide a more uniform and general treatment of these queries.
 
\begin{definition}\label{def:events}
An \hl{observational event} has the form \(\x\) where \(\X\) is a set of endogenous variables.
An \hl{interventional event} has the form \(\y_\x\) where \(\X\) and \(\Y\) are sets of endogenous variables.
A \hl{counterfactual event} has the form \(\ev_1, \ldots, \ev_n\) where \(\ev_i\) is an observational or interventional event.
\end{definition}
An observational event \(\x\) says that variables \(\X\) took the value \(\x\). 
For example, a patient did not take the treatment and died (\(\n x, \n y\)).
An interventional event \(\y_\x\) says that variables \(\Y\) took the value \(\y\) after setting variables \(\X\) to \(\x\) by an intervention.
For example, a patient survived after they were given the treatment (\(y_x\)).
A counterfactual event is a conjunction of events where each could be observational or interventional.
For example, a patient who responds to treatment did not take it and died (\(y_x, \n x, \n y\)).

\begin{definition}\label{def:worlds}
A \hl{world} for SCM \(\scm\) is an instantiation of its exogenous variables. 
The worlds of an observational event \(\ws_\scm(\x)\) are those worlds that fix the values of variables \(\X\) to \(\x\).
The worlds of an interventional event \(\ws_\scm(\y_\x)\) are defined as \(\ws_{\scm_\x}(\y)\).
The worlds of a counterfactual event \(\ws_\scm(\ev_1, \ldots, \ev_n)\) are defined as \(\ws_\scm(\ev_1) \cap \ldots \cap \ws_\scm(\ev_n)\). 
\end{definition}

The following table, borrowed from~\cite{Bareinboim20211OP}, shows all sixteen worlds of the above SCM \(\scm\), 
together with their probabilities and the unique states they entail for endogenous variables.
\begin{center}
\includegraphics[width=.70\textwidth]{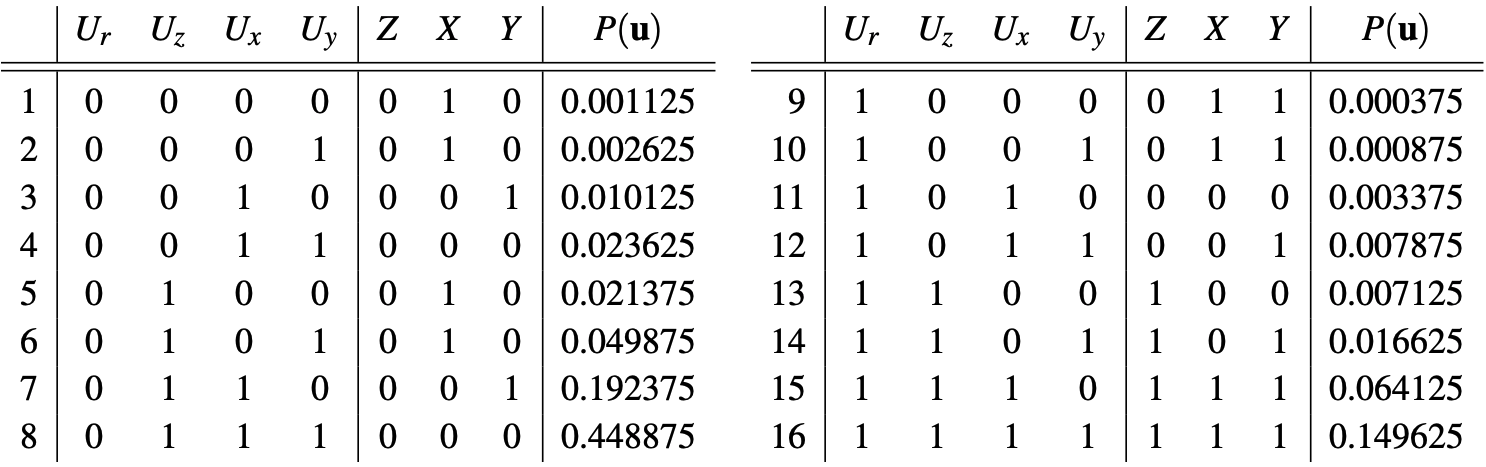}
\end{center}
Using this table and a similar one for sub-model \(\scm_x\), we obtain
 \(\ws_\scm(\n x, \n y) = \{\u_4,\u_8, \u_{11}, \u_{13}\}\) and
\(\ws_\scm(y_x) = \{\u_9,\ldots,\u_{16}\}\) which leads to
\(\ws_\scm(y_x, \n x, \n y) = \ws_\scm(\n x, \n y)  \cap \ws_\scm(y_x) =  \{\u_{11}, \u_{13}\}\).

We are now ready to define the probability of any SCM event.

\begin{definition} \label{def:queries}
Let \(\scm\) be an SCM with exogenous variables \(\U\) and distributions \(\pr(U)\) for \(U \in \U\).
The probability of event \(\ev\) with respect to SCM \(\scm\) is defined as:
\begin{equation}
\pr(\ev) = \sum_{\u \in \ws_\scm(\ev)} \pr(\u).
\label{eq:prq}
\end{equation}
\end{definition}
Hence, \(\pr(\n{x},\n{y}) = \pr(\u_4)+\pr(\u_8)+\pr(\u_{11})+\pr(\u_{13}) = 0.4830\) and
\(\pr(y_x, \n x, \n y) = 0.0105\).
We can now compute the probability that a patient who did not take the treatment and died would have been alive had they been given the treatment,
\(\pr(y_x | \n x, \n y) = \pr(y_x, \n x, \n y)  / \pr(\n{x},\n{y}) = 0.0217\). 
We will focus next on associational and interventional queries, leaving counterfactuals to future work. 

\section{Causal Inference Using Circuits}
\label{sec:ci}

As we just saw, one can answer sophisticated causal queries based on a fully specified SCM. 
However, such a model may not be available, particularly the identities of causal mechanisms and the distributions over exogenous variables.
What is more common is to have the causal graph of an underlying SCM in addition to observational data about endogenous variables.
A key task of causal inference is then to draw conclusions based on this limited input,
particularly about interventional probabilities which is the task we shall focus on. 
We will next show how this cross-layer inference can be realized using feed-forward circuits that are compiled from the non-parametric causal
graphs of SCMs. The discussion will reveal the significance of compiling the smallest possible circuit for a causal graph.
It will also motivate the new compilation algorithm in~\cite{DarwicheECAI20b} which is particularly relevant 
to the causal graphs of SCMs (in contrast to Bayesian networks). We shall discuss and study further this algorithm 
in Sections~\ref{sec:ve}-\ref{sec:vec} where we will also contribute to understanding its complexity.

\begin{wrapfigure}[10]{r}{0.40\textwidth}
\centering
\vspace{-3mm}
  \includegraphics[width=\linewidth]{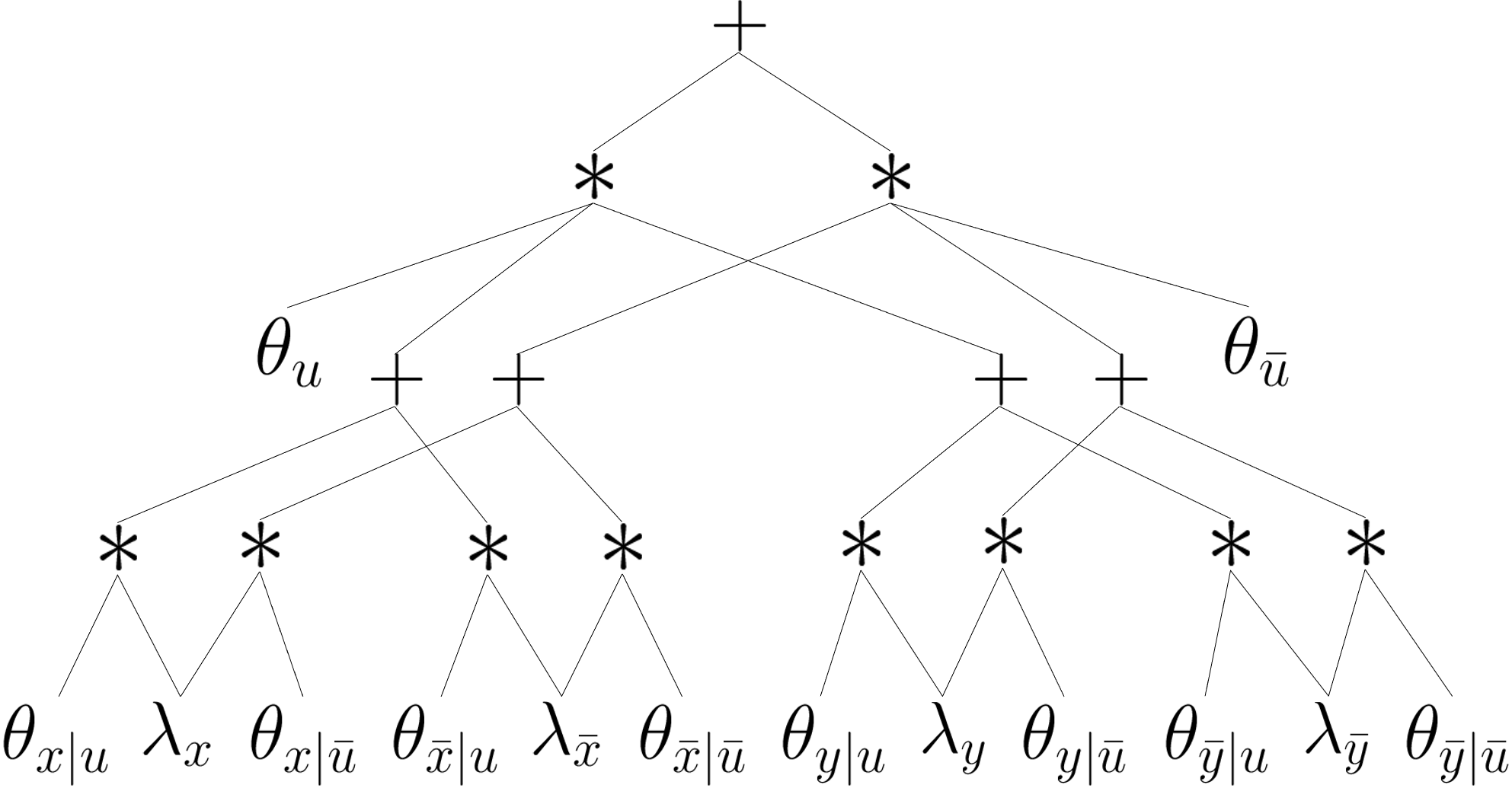}
  \caption*{\small arithmetic circuit (AC)}
\end{wrapfigure}
\textbf{The Circuits of Causal Graphs\ \ }
Consider the causal graph \(X \leftarrow U \rightarrow Y\) over binary variables and its compiled arithmetic 
circuit (AC) shown on the right.
This circuit has two types of inputs:
{\em symbolic parameters} \(\theta\) and {\em indicators} \(\lambda\). 
There are two parameters for exogenous variable \(U\) (\(\theta_u\) and \(\theta_{\n u}\)) which specify its distribution.
There are four parameters for endogenous variable \(X\) (\(\theta_{x|u}, \theta_{\n x | u}, \theta_{x | \n u}, \theta_{\n x | \n u}\))
which specify its causal mechanism. The remaining four parameters specify the mechanism for endogenous variable \(Y\). 
The indicators correspond to the values of endogenous variables. Variable \(X\) has indicators \(\lambda_x\) and \(\lambda_{\n x}\) and
variable \(Y\) has indicators \(\lambda_y\) and \(\lambda_{\n y}\). 
This circuit can compute the probability of any observational event \(\ev = \x\) in time linear in the circuit size. 
We simply set each indicator to \(1\) if its subscript is
compatible with the event \(\eta\), otherwise we set it to \(0\), and then evaluate the circuit~\cite{DarwicheJACM03}. 
For the event \(\ev = \n x\), the indicators are set to
\(\eql(\lambda_{x},0), \eql(\lambda_{\n x},1), \eql(\lambda_{y},1), \eql(\lambda_{\n y},1)\). 
If we (symbolically) evaluate the circuit under this indicator setting, we get 
\(\theta_u \theta_{\n x | u} (\theta_{y | u} + \theta_{\n y | u}) +  \theta_{\n u} \theta_{\n x | \n u} (\theta_{y | \n u} + \theta_{\n y | \n u})\) 
which is the expected result, \(\pr(\n x)\).
The circuit can also compute the probability of any interventional event \(\eta = \y_\x\) in time linear in the circuit size. 
This is remarkably simple as well. To compute \(\pr(\y_\x)\), known as the {\em causal effect,} we first
set the parameters of all variables in \(\X\) to \(1\) and then evaluate the circuit at instantiation \(\x,\y\).\footnote{This method
follows directly from the mutilation semantics of interventions~\cite[Section 1.3.1]{pearl00b} and the polynomial
semantics of arithmetic circuits~\cite{DarwicheJACM03,ChoiDarwiche17}. It is a 
slight variation on~\cite{ijcai/Qin15} which also emulates interventions by adjusting model parameters; see also~\cite{ijcai/WangLK21}.}
Suppose that \(\x = \n x\) and \(\y = \n y\) in our running example. To compute \(\pr({\n y}_{\n x})\),
we evaluate the circuit while setting the parameters for variable \(X\) to
\(\theta_{x|u} = \theta_{\n x | u} =  \theta_{x | \n u} =  \theta_{\n x | \n u} = 1\), and setting
the indicators to
\(\eql(\lambda_{x},0), \eql(\lambda_{\n x},1), \eql(\lambda_{y},0), \eql(\lambda_{\n y},1)\).
If we  (symbolically) evaluate the circuit under these settings, we get
\(\pr({\n y}_{\n x}) = \theta_u \theta_{\n y | u} + \theta_{\n u} \theta_{\n y | \n u} = \pr(\n y)\) 
which is the expected result  (\(X\) has no causal effect on \(Y\)).

\textbf{Exploiting Unknown Mechanisms\ \ }
Consider the endogenous variable \(X\) in the causal graph we just discussed and its 
parameters \(\theta_{x|u}, \theta_{\n x | u}, \theta_{x | \n u}, \theta_{\n x | \n u}\).
Since these parameters specify a causal mechanism, they must all be in \(\{0,1\}\) subject to the
constraints \(\theta_{x|u}+ \theta_{\n x | u}=1\) and \(\theta_{x | \n u}+ \theta_{\n x | \n u}=1\).
The main contribution of the new compilation algorithm in~\cite{DarwicheECAI20b} is that it can exploit these constraints---without needing to know
the specific values of \(\theta_{x|u}, \theta_{\n x | u}, \theta_{x | \n u}, \theta_{\n x | \n u}\)---to produce
circuits whose size can be exponentially smaller compared to methods
developed during the last two decades; see,~e.g.,~\cite{DarwicheJACM03,ijcai/ChaviraD05,Chavira.Darwiche.Ijcai.2007,ecsqaru/ChoiKD13,ShenCD16}
and~\cite[Chapters 12,13]{Darwiche09}. 
These earlier methods can exploit functional dependencies computationally but only if they know the specific values of parameters.
This does not help in a causality context (or a learning context more generally) where we
do not know these values.  
The details of how the algorithm does this are discussed in Sections~\ref{sec:ve}-\ref{sec:vec}.
\begin{wrapfigure}[6]{r}{0.4\textwidth}
\centering
\vspace{-3mm}\includegraphics[width=1\linewidth]{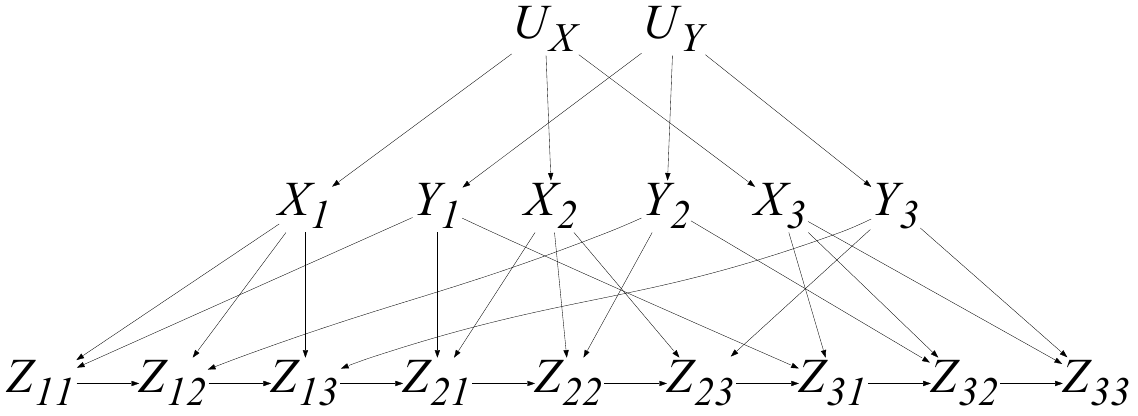}
  \caption*{}
\end{wrapfigure}
For now, consider the family of causal graphs on the right which generally has variables
\(U_X, U_Y, X_i, Y_j, Z_{ij}\) for \(i, j = 1, \ldots, n\). These models have treewidth \(\geq n+1\)
and hence are not accessible to non-parametric inference methods as they take time exponential in \(n\).
Moreover, the causal effect of \(X_1\) on \(Z_{nn}\) has the only back-door \(X_2, \ldots, X_n\).
By exploiting (unknown) causal mechanisms, we can now compile these causal
graphs into circuits of size \(O(n^2)\) (linear in the number of variables), allowing us
to compute associational and interventional queries in \(O(n^2)\) time. We will say more about this model and back-doors later.

\textbf{Cross-Layer Inference\ \ }
We next show how to perform cross-layer causal inference using circuits. 
In a nutshell, we will use the circuit to estimate maximum-likelihood parameters for both exogenous and endogenous variables in the causal graph.
We will then plug the estimates into the circuit and use it to compute associational and interventional probabilities in time linear in the circuit size
as shown earlier.
Suppose \(\V\) is the set of observed endogenous variables.
Since the causal graph has hidden variables, the maximum-likelihood parameters are not unique so they are not identifiable. 
However, the distribution \(\pr(\V)\) is identifiable in this case. That is, even though we may have
multiple maximum-likelihood estimates, they all lead to the same distribution \(\pr(\V)\).\footnote{Ying Nian Wu provided 
the following argument for infinite data. Let \(\pr_D(\V)\) be the data distribution
and \(\pr_\theta(\U,\V)\) be the model so \(\pr_\theta(\V) = \sum_{\U} \pr_\theta(\U,\V)\). 
Maximum-likelihood estimation is equivalent to minimizing the 
KL divergence \(KL(\pr_D(\V) | \pr_\theta(\V))\). If \(\theta\) is not identifiable, then all solutions of \(\theta\) belong to an equivalence 
class that minimizes the KL divergence and they all give the same marginal \(\pr_\theta(\V)\).}
This approach will need to be applied carefully though as its validity depends on (1)~the specific query of interest,
(2)~the set of observed variables and (3)~the causal graph structure. We will discuss this
in detail after elaborating further on the estimation of maximum-likelihood parameters.

\textbf{Estimation\ \ }
Since each example in a dataset corresponds to an observational event,
one can use arithmetic circuits to compute the likelihood function by simply evaluating
the circuit at each example (in linear time) and then multiplying the results. 
The algorithm in~\cite{DarwicheECAI20b} facilitates this computation further as it compiles causal graphs into circuits in the form of {\em tensor graphs.}
These are computation graphs in which nodes represent tensor operations instead of arithmetic operations, allowing one 
to evaluate and differentiate the circuit significantly more efficiently (think of a tensor operation as doing a bulk of arithmetic 
operations in parallel). 
Tensor graphs can lead to orders of magnitude speedups in estimation and inference time, especially that they allow batch (parallel) processing
of examples---see~\cite{DarwicheECAI20b,DBLP:PGM20a}
which compiled circuits with tens of millions of nodes, leading to evaluation times in milliseconds.
Backpropagation on arithmetic circuits takes time linear in the circuit size. Moreover, the partial derivatives
with respect to circuit parameters correspond to marginals over families in the causal graph (nodes and their parents)~\cite{DarwicheJACM03}, 
which is all that one needs to compute parameter updates for the EM algorithm; see, for example,~\cite[Eq.~17.7]{Darwiche09}.
Hence, one can use methods such as gradient descent and EM to seek maximum-likelihood estimates, but
the efficacy of these methods needs further investigation under the stated conditions including finite data.

\textbf{Identifiability\ \ }
A central question in causal inference is whether interventional probabilities can be identified based on a causal graph and
(infinite) observational data on the endogenous variables~\(\V.\)
Intuitively, identifiability means that interventional probabilities can be computed using any parameterization of the causal graph (or circuit)
that yields the true marginal distribution \(\pr(\V)\); see~\cite[Def~3.2.4]{pearl00b} for a formal definition.
A well behaved case arises when each exogenous variable feeds into at most one causal mechanism. 
These models are said to be {\em Markovian} and the case is termed {\em no unobserved confounders.}
Interventional probabilities are always identifiable for Markovian models, so we can always compute the causal effect \(\pr(\y_\x)\) for
these models using circuits parameterized by maximum-likelihood parameters.\footnote{See~\cite[Corollary 4]{docalculus} for a weaker, necessary
condition that guarantees identifiability of all causal effects.}
\begin{wrapfigure}[5]{r}{0.15\textwidth}
\centering
\vspace{-2mm}  \includegraphics[width=\linewidth]{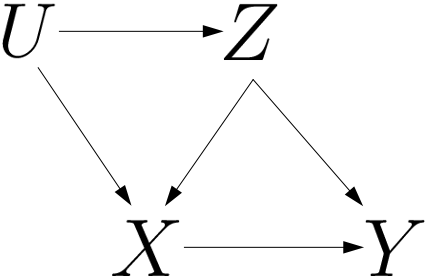}
  \caption*{}
\end{wrapfigure}
If an exogenous variable feeds into more than one causal mechanism, the model is said to be {\em semi-Markovian}.
In this case, interventional probabilities are identifiable only when certain conditions are met. 
The causal graph on the right corresponds to a semi-Markovian model since \(U\) feeds into the causal mechanisms for both \(X\) and \(Z\).
The causal effect \(\pr(y_x)\) is identifiable since \(\pr(y_x) = \sum_z \pr(y | x, z) \pr(z)\).
However, the causal effect \(\pr(x_z)\) is not identifiable as it cannot be uniquely determined
based on the causal graph and the distribution \(\pr(X,Y,Z)\).
The {\em do-calculus} provides a complete and efficient characterization of identifiable, interventional probabilities
based on observational data~\cite{do-calculus-95,aaai/TianP02,docalculus}; see also~\cite{aaai/HuangV06}.\footnote{See 
also~\cite{uai/LeeCB19} for a treatment of identifiability based on both observational and interventional data and~\cite{pnas/BareinboimP16} 
for a survey that considers other tasks such as the evaluation of soft interventions.}
It is based on a set of rules that can be used to transform an interventional probability into a formula that includes only 
associational probabilities (we will call this an {\em identifiability formula}). 
If the rules fail to make such a derivation, then the interventional probability is not identifiable. 
Other methods such as {\em back-door}~\cite{backdoor} and {\em front-door}~\cite{uai/PearlR95} provide simpler but incomplete
tests and lead to simple identifiability formulas.
For example, the back-door and front-door formulas have the forms \(\pr(\y_\x) = \sum_\z \pr(\y|\x,\z)\pr(\z)\) and
\(\pr(\y_\x) = \sum_\z \pr(\z|\x)\sum_{\x'} \pr(\y|\x',\z)\pr(\x')\), where \(\Z\) is called a back-door or front-door, respectively.
A more refined identifiability test that utilizes context-specific independence relations was proposed recently~\cite{nips/TikkaHK19}, 
thus expanding the reach of cross-layer causal inference. 
These relations correspond to equating certain parameters and can be integrated into circuits~\cite{ijcai/ChaviraD05}
to reduce the number of estimands.
In summary, we can compute causal effects on semi-Markovian models using circuits parameterized by maximum-likelihood estimates,
but only for identifiable queries as licensed by the do-calculus or a more refined identifiability procedure.  
 
\textbf{Identifiability Formulas vs Circuits\ \ }
Most identifiability procedures yield formulas (also called the {\em effect estimand}) which play three roles:
they provide a proof of identifiability; they point to endogenous variables whose measurement guarantees identifiability; 
and they allow one to evaluate the causal effect by estimating quantities that populate such formulas.
Back-door and front-door formulas have simple forms (albeit exponential sums) but
the do-calculus and the procedure in~\cite{nips/TikkaHK19} may generate identifiability formulas that are much more complex. 
Some procedures do not even aim to produce identifiability formulas; e.g.,~\cite{jair/Halpern00}.
To evaluate causal effects using circuits, one only needs an identifiability test not a formula.
That is, one estimates circuit parameters only once and then uses the parametrized circuit
to answer any identifiable query in time linear in the circuit size---regardless of how complex the identifiability
formula may be and without needing to have access to one.\footnote{Using circuits in this manner requires fixing the cardinality
of variables including exogenous ones.}
Additional knowledge such as context-specific independence and known mechanisms can be directly integrated
into the circuit~\cite{ijcai/ChaviraD05}, which can only improve the quality of estimates under finite data. 
At its core, this use of circuits amounts to computing causal effects based on the classical method of mutilating causal graphs,
armed by an observation and an advance. The observation is that using maximum-likelihood parameters is sound
if the causal effect is identifiable. The advance is that we can now perform this computation much more efficiently due to
exploiting unknown mechanisms. 

\textbf{The Circuit Compilation Process\ \ }
We will next discuss the circuit compilation algorithm introduced recently in~\cite{DarwicheECAI20b}. 
We will focus on the key insights behind the algorithm and slightly adjust it to suit our current objectives (the original algorithm 
targetted specific queries as is typically demanded in a supervised learning setting).
In a nutshell, the algorithm is based on the classical algorithm of {\em variable elimination} (\VE) with two
exceptions. First, we will use \VE\ symbolically by working with symbolic
parameters instead of numeric ones. Second, we will empower \VE\ by two new theorems based
on unknown causal mechanisms which can reduce its complexity exponentially.
We will review VE in Section~\ref{sec:ve} and then present the new 
theorems and compilation algorithm in Sections~\ref{sec:cm} and~\ref{sec:vec}.

\section{The Variable Elimination Algorithm (\VE)}
\label{sec:ve}

\VE\ operates on causal graphs which are parameterized by factors.
A {\em factor} over variables \(\X\) is a function \(f(\X)\) that maps each instantiation \(\x\) into a number \(f(\x)\).
For each node \(X\) and its parents \(\PA\) in the
\begin{wrapfigure}[5]{r}{0.4\textwidth}
\centering
\vspace{-5mm}
\begin{center}
\scalebox{0.77}{
\(
\begin{array}{c|rl}
U & f(U) & \\ \hline
u_0 & 0.3 & \theta_{u_0} \\
u_1 & 0.1 & \theta_{u_1} \\
u_2 & 0.6 & \theta_{u_2}  
\end{array}
\quad
\begin{array}{cc|cc}
X & Y & g(XY) & \\ \hline
x_0 & y_0 & 0 & \theta_{y_0 | x_0} \\
x_0 & y_1 & 1 & \theta_{y_1 | x_0} \\
x_1 & y_0 & 1 & \theta_{y_0 | x_1} \\
x_1 & y_1 & 0 & \theta_{y_1 | x_1} 
\end{array}
\)
}
\end{center}
  \caption*{}
\end{wrapfigure}
causal graph, we need a factor \(f_X(X,\PA)\) where \(f_X(x,\pa) = \pr(x|\pa)\).
For example, factor \(f(U)\) on the right specifies 
the distribution \(\pr(U)\) for exogenous variable \(U\) and factor \(g(XY)\)
specifies the mechanism for endogenous variable \(Y\): \(x_0 \mapsto y_1\), \(x_1 \mapsto y_0\).
\VE\ is based two factor operations: multiplication and sum-out.
The {\em product} of factors \(f(\X)\) and \(g(\Y)\) is another factor \(h(\Z)\),
where \(\Z = \X \cup \Y\) and \(h(\z) = f(\x)g(\y)\) for the unique instantiations \(\x\) and \(\y\) that are compatible with instantiation~\(\z\). 
{\em Summing-out} variables \(\Y \subseteq \X\) from factor \(f(\X)\) yields another factor \(g(\Z)\),
where \(\Z= \X \setminus \Y\) and \(g(\z) = \sum_\y f(\y\z)\). 
We use \(\sum_\Y f\) to denote the resulting factor~\(g\). 
 
The {\em joint distribution} of a parametrized causal graph is simply the product of its factors.
The causal graph in Figure~\ref{fig:jt}(b)  has factors \(f_A(A)\), \(f_B(AB)\), \(f_C(AC)\), \(f_D(BCD)\) and \(f_E(CE)\).
Its joint distribution is \(\pr(ABCDE) = f_A f_B f_C f_D f_E\).
To record an observation \(\eql(X,x)\), we use an auxiliary {\em evidence factor} \(\lambda_X(X)\) 
with \(\lambda_X(x)=1\) and \(\lambda_X(x')=0\) for \(x' \neq x\). 
A {\em posterior distribution} is obtained by normalizing the product of all factors in the causal graph
including evidence factors. 
Suppose we have evidence \(\e\) on variables \(A\) and \(E\) in Figure~\ref{fig:jt}(b). 
The posterior \(\pr(D | \e)\) is obtained by evaluating then normalizing the expression
\(
\sum_{ABCE} \lambda_A \lambda_E f_A f_B f_C f_D f_E.
\)
\VE\ tries to evaluate such expressions  
efficiently~\cite{zhangJAIR96a,dechterUAI96} based on two theorems; see, e.g., \cite[Chapter 6]{Darwiche09}.

The first theorem allows us to sum out variables in any order.
The second theorem allows us to pull out factors from sums, which can lead to exponential savings in time and space.
\begin{thm}\label{theo:ve0}
\(\sum_{\X\Y} f = \sum_\X \sum_\Y f = \sum_\Y \sum_\X f\).
\end{thm}
\begin{thm}\label{theo:ve1}
If variables \(\X\) appear in factor \(f\) but not in factor \(g\), then \(\sum_\X f \cdot g = g \sum_\X f\).
\end{thm}

Consider the expression \(\sum_{ABDE} f(ACE) f(BCD)\). A direct evaluation
multiplies the two factors to yield \(f(ABCDE)\) then sums out variables \(ABDE\). Using
Theorem~\ref{theo:ve0}, we can arrange the expression into \(\sum_{AE} \sum_{BD} f(ACE) f(BCD)\).
Using Theorem~\ref{theo:ve1}, we can arrange it further into \(\sum_{AE} f(ACE) \sum_{BD} f(BCD)\)
which is more efficient to evaluate. If we eliminate all variables using order \(\pi\), and if the largest
factor constructed in the process has \(w+1\) variables, then \(w\) is called the {\em width} of order \(\pi\).
The smallest width attained by any elimination order corresponds to the {\em treewidth} of the causal graph. 
The best time complexity that can be attained by \VE\ is 
\(O(n \exp(w))\), where \(n\) is the number of variables and \(w\) is the causal graph treewidth. 
This holds for any other non-parametric method known today (i.e., inference methods that do not exploit the graph 
parameters).

\section{Variable Elimination with Causal Mechanisms}
\label{sec:cm}

We next present two recent results that allow us to simplify expressions beyond what is permitted by 
Theorems~\ref{theo:ve0} and~\ref{theo:ve1}, leading to a tighter complexity based on what we shall
call the {\em causal treewidth.} 
We will use \(\FF\), \(\GG\), \(\HH\) to denote sets of factors, where each set is interpreted as a product of its factors.

\begin{definition}\label{def:mechanism}
A factor \(f(X,\PA)\) is said to be a \hl{mechanism} for variable \(X\) iff 
all numbers in the factor are in \(\{0,1\}\) and \(\sum_x f(x,\pa) = 1\) for every instantiation \(\pa\).
\end{definition}

\begin{thm}[\cite{DarwicheECAI20b}]\label{theo:ve2}
Let \(f\) be a mechanism for variable \(X\). If  \(f \in \GG\) and \(f \in \HH\), then
\(\GG \cdot \HH = \GG \sum_X \HH\). 
\end{thm}
According to this result, if a mechanism for \(X\) appears in both parts of a product, then
variable \(X\) can be summed out from one part without changing the value of the product.
This has a key corollary.

\begin{corollary}\label{coro:ve2}
Let \(f\) be a mechanism for  \(X\). If  \(f \in \GG\) and \(f \in \HH\), then
\(\sum_X \GG \cdot \HH = \left(\sum_X \GG\right) \left(\sum_X \HH\right)\). 
\end{corollary}
That is, if a mechanism for \(X\) appears in both parts of a product,
we can sum out variable \(X\) from the product by independently summing it out from each part.
This is a remarkable addition to the algorithm of variable elimination which has been under
study for a few decades now.
Corollary~\ref{coro:ve2} may appear unusable as it is predicated on
multiple occurrences of a mechanism whereas the factors of a causal graph
contain a single mechanism for each endogenous variable. This is where the second result comes in:
{\em replicating} mechanisms in a product does not change the product value. 

\begin{thm}[\cite{DarwicheECAI20b}]\label{theo:ve3}
For mechanism \(f\), if \(f \in \GG\), then  \(f \cdot \GG=\GG\).
\end{thm}

For an example that uses these theorems, 
consider the expression \(\alpha = \sum_X f(XY) g(XZ) h(XW)\). \VE\ has to multiply all three factors before 
summing out variable \(X\), leading to a factor over four variables \(XYZW\).
However, if factor \(f\) is a mechanism for variable \(X\), then we can replicate it by Theorem~\ref{theo:ve3}:
\(\alpha =  f(XY) g(XZ) f(XY) h(XW)\). Corollary~\ref{coro:ve2} then gives
\(\alpha = \sum_X f(XY) g(XZ) \sum_X f(XY) h(XW)\). Hence, we can now evaluate expression \(\alpha\)
without having to construct any factor over more than three variables. This technique can
more generally lead to exponential savings
since the size of a factor is exponential in the number of its variables.

We will refer to the extension of \VE\ with Theorems~\ref{theo:ve2} and~\ref{theo:ve3} 
as \VEC\ ({\bf V}ariable {\bf E}limination for {\bf C}ausality). 
We emphasize that these new theorems do not require the values of parameters (i.e., specific mechanisms).
They only need to know if a variables is functionally determined by its parents.  

\section{Compiling Causal Graphs Into Circuits}
\label{sec:vec}

\def\shrinka{{\sc shrink\_sep}}
\def\fsum{{\sc sum}}
\def\fvars{\V}

We next show how \VE/\VEC\ can be used {\em symbolically} to compile non-parametric causal
graphs into arithmetic circuits with symbolic parameters.
We will first show this concretely on a small example using \VE, 
then discuss a general compilation algorithm based on \VE\ and finally based on \VEC. 

\begin{wrapfigure}[6]{r}{0.465\textwidth}
\centering
\begin{center}
\vspace{-3mm}
\hspace{-2mm}
\scalebox{0.9}{
\(
\begin{array}{c|c}
U & f(U) \\ \hline
u & \theta_{u} \\
\n u & \theta_{\n u}
\end{array}
\quad
\begin{array}{cc|cc}
U & V & g(UV) & \\ \hline
u & v        & \theta_{v | u} \\
u & \n v    & \theta_{\n v | u} \\
\n u & v    & \theta_{v | \n u} \\
\n u & \n v & \theta_{\n v | \n u} 
\end{array}
\quad
\begin{array}{c|c}
V & h(V) \\ \hline
v & \lambda_{v} \\
\n v & \lambda_{\n v}
\end{array}
\)
}
\end{center}
  \caption*{}
\end{wrapfigure}
Consider the causal graph \(U \rightarrow V\) with binary variables. 
The parameters of this graph are given by the factors \(f(U)\) and \(g(UV)\) shown on the right. 
We also added an evidence factor for endogenous variable \(V\) as it will be measured. 
These factors have symbolic parameters instead of numeric ones.
We will further overload the \(+\) and \(*\) operators so they now construct circuit nodes instead of performing numeric operations. 
That is, each entry in the above factors can be viewed as a leaf circuit node. 
When multiplying, say, node \(\theta_u\) with node \(\theta_{v | u}\), we construct a circuit
node with \(*\) as its label and nodes \(\theta_u\), \(\theta_{v | u}\) as its children.
And similarly for addition.
We can now get a circuit for the causal graph by multiplying all its factors, including evidence factors, and then summing out all variables,
\(AC = \sum_{UV} f(U)g(UV)h(V)\). The resulting factor \(AC\) will have a single entry which contains the root of compiled circuit 
\(
 \lambda_{v}*\theta_{u}*\theta_{v | u} +
\lambda_{\n v}*\theta_{u}*\theta_{\n v | u}  +
 \lambda_{v}*\theta_{\n u}*\theta_{v | \n u} +
\lambda_{\n v}*\theta_{\n u}*\theta_{\n v | \n u}
\).
The equivalent expression \(AC = \sum_V h(V) \sum_U f(U)g(UV)\) gives the circuit
\(
 \lambda_{v} * (\theta_{u}*\theta_{v | u} + \theta_{\n u}*\theta_{v | \n u}) +
\lambda_{\n v} * (\theta_{u}*\theta_{\n v | u}  +\theta_{\n u}*\theta_{\n v | \n u})
\).
Hence, the size and shape of a compiled circuit depend on how we schedule factor operations (multiplication and sum-out).
The symbolic use of \VE\ to compile circuits was initially proposed in~\cite{Chavira.Darwiche.Ijcai.2007}. 
This method was recently refined in~\cite{DarwicheECAI20b} by 
(1)~{\em scheduling} factor operations based on a specific class of {\em binary jointrees}~\cite{DBLP:conf/uai/Shenoy96} and 
(2)~allowing one to compile circuits using \VEC\ by {\em thinning}
the jointree. We will explain this advance after a brief review of (binary) jointrees.

\begin{figure}[tb]
\centering
\includegraphics[width=.8\textwidth]{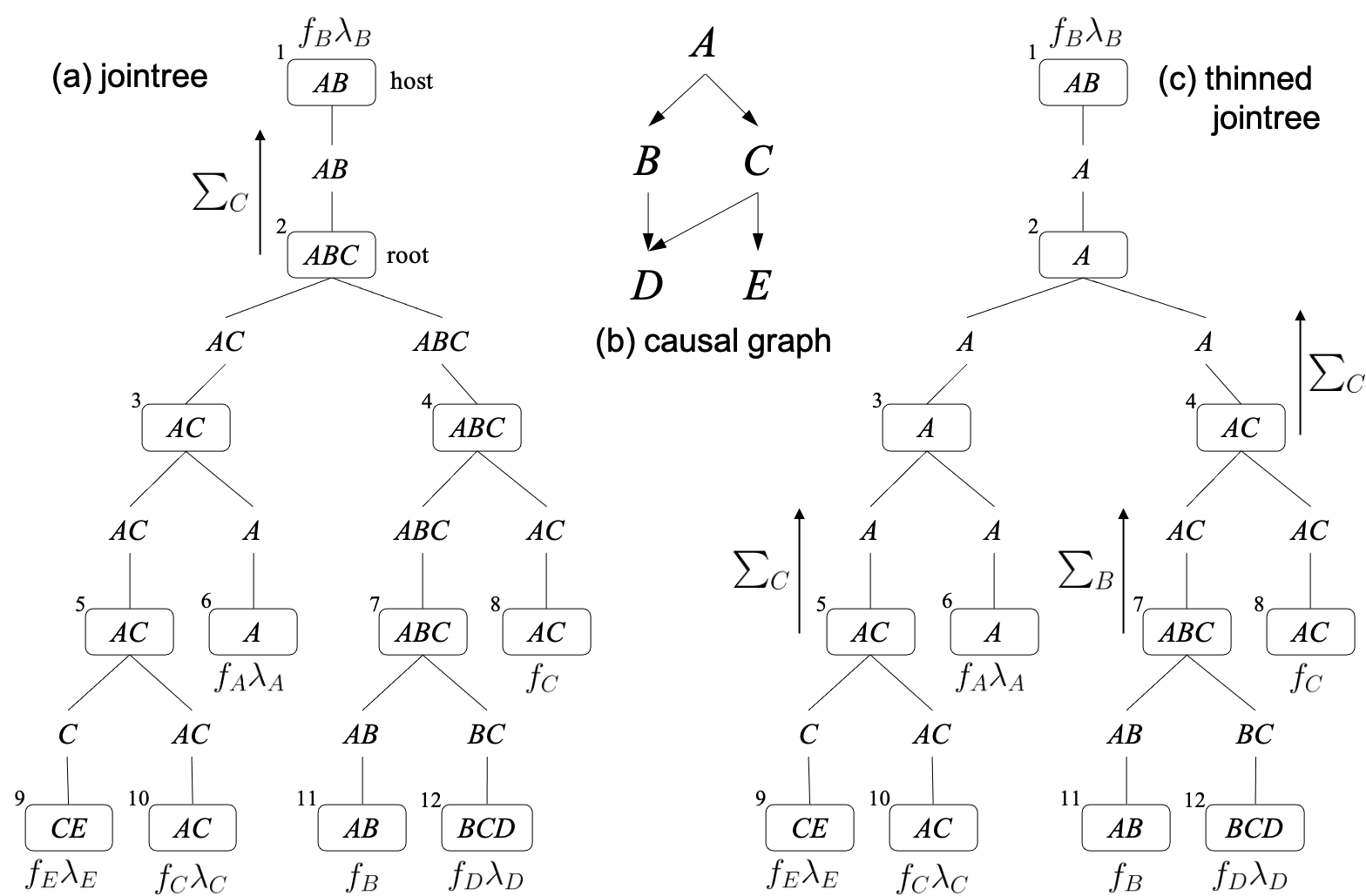}
\caption{A causal graph (middle) with a jointree (left) and a thinned jointree (right).
The mechanisms for variables \(B\) and \(C\), \(f_B\) and \(f_C\), are replicated twice in the jointrees.
\label{fig:jt}}
\end{figure}

\textbf{Jointrees\ \ }
A binary jointree is a tree in which each node is either a {\em leaf} (has a single neighbor) or {\em internal} (has three neighbors). 
We will require the leaf nodes to be in one-to-one correspondence with the factors of a causal graph,
including replicated factors but excluding evidence factors. As we show next, the topology of a binary jointree determines all its
properties, including the set of variables attached to each node, called a {\em cluster,} 
and the set of variables attached to each edge, called a {\em separator.}
Figure~\ref{fig:jt}(a) depicts a binary jointree for the causal graph in Figure~\ref{fig:jt}(b). 
Following the convention in~\cite{DarwicheECAI20b},
this jointree is layed out so that each internal node has two neighbors below it (children) and the third neighbor above it (parent).
The leaf nodes of this jointree are numbered \(1,6,8,9,10,11,12\) and correspond to the causal graph factors,
\(f_B, f_A, f_C, f_E, f_C, f_B, f_D\) (we replicated the factors for \(B\) and \(C\)).
The separator for edge \((i,j)\) between node \(i\) and its parent \(j\) is denoted \(\sep(i)\) and
contains variables that are shared between factors on both sides of the edge \((i,j)\). 
For example, \(\sep(3) = \{A,C\}\) as these are the variables shared 
between factors at leaves \(\{6,9,10\}\) and factors at leaves \(\{1,8,11,12\}\).
The cluster of node \(i\) is denoted \(\cls(i)\).
The cluster of a leaf node is the variables of its associated factor. The cluster of an internal node
is the union of separators connected to its children. 
For example, for leaf node \(9\) with factor \(f_E\), \(\cls(9) = \vars(f_E) = \{C,E\}\). 
Moreover, for internal node \(7\) with children \(11\) and \(12\), \(\cls(7) = \sep(11) \cup \sep(12) = \{A,B,C\} \).

\textbf{Scheduling\ \ }
Given a binary jointree, \VE\ (and later \VEC) schedules its operations as follows.
Visiting nodes bottom-up in the jointree, each node \(i\) computes a factor \(f(i)\) and sends it to its parent.
A leaf node \(i\) computes \(f(i)\) by projecting its associated factor on \(\sep(i)\). 
For example, \(f(9) = \sum_{E} f_E \lambda_E\).
An internal node \(i\) computes \(f(i)\) by multiplying the factors it receives from its children and then projecting the product on \(\sep(i)\).
For example, \(f(2) = \sum_C f(3) f(4)\).
This process terminates at the top leaf node \(r\) which multiplies
the factor it receives from its single child \(c\) with its own factor and then projects the product on the empty set.
In Figure~\ref{fig:jt}(a), the final computation is \(\sum_{AB} f_B \lambda_B f(2)\).
Denoting the factor at leaf node \(i\) by \(\FF_i\), this process yields the factor 
\(AC = \sum_{\cls(r)} \FF_r f(c)\), where
\[
\scalemath{0.91}{
f(i) = \sum_{\cls(i)\setminus\sep(i)} \FF_i  \quad\mbox{if \(i\) is leaf}; \quad
f(i) = \sum_{\cls(i)\setminus\sep(i)} f(c_1) f(c_2) \quad\mbox{if \(i\) has children \(c_1,c_2\).}
}
\]
The final factor \(AC\) has a single entry which contains the root of the compiled circuit as shown earlier.
Moreover, the size of this circuit is determined by the jointree clusters and separators. 
Each cluster/separator contributes a number of multiplication/addition nodes that is exponential in the cluster/separator size.
In a jointree, the largest cluster dominates the largest separator and the size
of the largest cluster minus \(1\) is called the jointree {\em width.} 
Furthermore, the smallest width attained by any jointree corresponds to the treewidth of the causal graph; see,~\cite[Chapter 9]{Darwiche09}. 
Hence, the complexity of this compilation method is exponential in the treewidth of the causal graph.

\textbf{Thinning\ \ }
This complexity was recently significantly improved by exploiting (unknown) causal mechanisms~\cite{DarwicheECAI20b}.
The basic idea is to {\em thin} the jointree by shrinking its separators (and hence clusters) while
maintaining the correctness of compiled circuit. The thinning process is based on
Theorems~\ref{theo:ve2} and~\ref{theo:ve3} and can lead to an exponential reduction in the circuit size.
To see the key insight behind this thinning process, consider node \(2\) in the jointree of Figure~\ref{fig:jt}(a).
The separators \(\sep(3)\) and \(\sep(4)\) of its children both contain variable \(C\).
Hence, the factors \(f(3)\) and \(f(4)\) sent by these children to node~\(2\) both contain variable \(C\).
Since \(C\) does not appear in \(\sep(2)\) it gets summed out at node \(2\) so it does not 
appear in the factor \(f(2)\) that this node sends to its parent.
However, since we have two replicas of the mechanism for variable \(C\) at leaf nodes \(8\) and \(10\)
(as licensed by Theorem~\ref{theo:ve3}), we can sum out \(C\) earlier, at nodes \(4\) and \(5\)
(as licensed by Theorem~\ref{theo:ve2}). This means that \(C\) can be removed from 
\(\sep(4)\), \(\sep(5)\) and also \(\sep(3)\). We can similarly sum out variable \(B\)
at node \(7\), which removes it from \(\sep(7)\), \(\sep(4)\) and \(\sep(2)\). The shrinking of 
separators causes clusters to shrink as well, leading to the thinned jointree in 
Figure~\ref{fig:jt}(c) and a corresponding smaller circuit compilation.

\begin{figure}[tb]
\centering
\includegraphics[width=.74\textwidth]{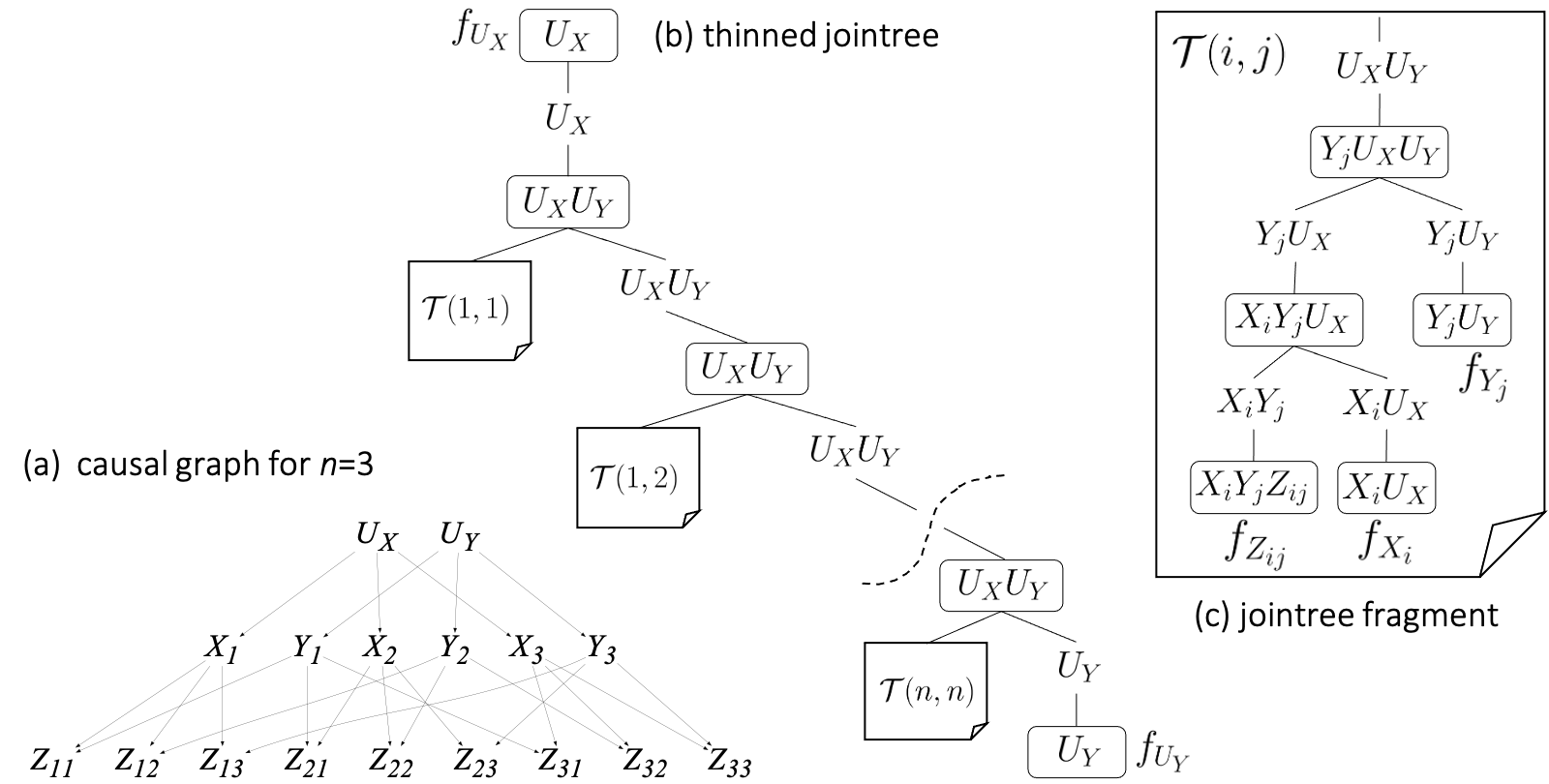}
\caption{A causal graph and its thinned jointree. Mechanisms \(f_{X_i}\) and \(f_{Y_j}\) are replicated \(n\) times.
\label{fig:ctw}}
\end{figure}

\textbf{Causal Treewidth\ \ }
The attained reduction in complexity depends on 
(1) the number of replicas for each mechanism;
(2) the used binary jointree;
and
(3) how the jointree is thinned.
Corresponding heuristics were proposed in~\cite{DarwicheECAI20b} and the
resulting algorithm was shown to yield exponential reductions in the size of compiled
circuits on a number of benchmarks (elimination orders and jointrees are also constructed using
heuristics since finding optimal ones is NP-hard). 
This motivates a new measure of complexity which we call
the {\em causal treewidth.} We define this as the smallest width attained by any thinned 
jointree for a given causal graph. We next complement the empirical findings in~\cite{DarwicheECAI20b}
by showing that causal treewidth dominates treewidth and can be bounded when the treewidth is not. 

\begin{thm}\label{theo:ctw}
The causal treewidth is no greater than treewidth. Moreover,
there is a family of causal graphs with \(n^2+2n+1\) variables, treewidth \(n+1\) and causal treewidth \(2\)
where \(n\) is an integer \(\geq 1\).
\end{thm}

\begin{proof}[Proof Sketch]
Consider a causal graph \(\cg\) with treewidth~\(w\). Without replicating mechanisms, we
can always get a (thinned) jointree with width \(w\). Hence, the causal treewidth of \(\cg\) is~\(\leq w\).
To show the second part of the theorem,
consider the family of causal graphs \(\cg_n\) with exogenous variables \(U_X\), \(U_Y\) and endogenous variables 
\(X_i\), \(Y_j\), \(Z_{ij}\) for \(i, j=1, \ldots, n\) (\(2+2n+n^2\) variables), and edges
\(U_X \rightarrow X_i\), \(U_Y \rightarrow Y_j\), \(X_i \rightarrow Z_{ij}\), \(Y_j \rightarrow Z_{ij}\).
Figure~\ref{fig:ctw}(a) depicts \(\cg_3\).
We next show that \(\cg_n\) has treewidth \(n+1\) based on standard 
techniques for treewidth; see, e.g.,~\cite[Chapter 9]{Darwiche09}. 
The moral graph of \(\cg_n\) is obtained by dropping edge directions and
connecting every pair of nodes \(X_i\) and \(Y_j\) by an undirected edge. 
Each \(Z_{ij}\) has only two (connected) neighbors in the moral graph (\(X_i\) and \(Y_j\))
so it is a simplicial node. Hence, there must exist an optimal elimination order that starts with
nodes \(Z_{ij}\)~\cite[Section 9.3.2]{Darwiche09}.
After eliminating all \(Z_{ij}\), nodes \(U_X\) and \(U_Y\) will each have \(n\) neighbors,
and nodes \(X_i\) and \(Y_j\) will each have \(n+1\) neighbors.
A simple argument shows that eliminating these variables in any order from the moral graph will create
a clique over \(n+1\) variables so the treewidth is \(\geq n+1\). One can easily verify that 
the elimination order \(Z_{11},\ldots,Z_{nn},U_X,Y_1,\ldots,Y_n,U_Y,X_1,\ldots,X_n\)
has width \(n+1\) so the treewidth of \(\cg_n\) is \(n+1\).
Figure~\ref{fig:ctw}(b) depicts a thinned jointree for \(\cg_n\) with width \(2\), 
which results from cascading \(n^2\) instances of the jointree fragment \({\cal T}(i,j)\) in Figure~\ref{fig:ctw}(c). 
This fragment contains the mechanism for \(Z_{ij}\) and replicas of the mechanisms for \(X_i\) and \(Y_j\). 
This thinned jointree is optimal since the mechanism for \(Z_{ij}\) contains \(3\) variables
so any thinned jointree must have a cluster of size \(\geq 3\). Hence, 
the causal treewidth of \(\cg_n\) is \(2\).
\end{proof}
Theorem~\ref{theo:ctw} effectively says that circuits compiled by \VEC\ are no larger than those compiled by \VE\
and can be exponentially smaller. We finally note that a variation \(\cg'_n\) on \(\cg_n\) was shown in
Section~\ref{sec:ci} with additional edges between variables \(Z_{ij}\). The treewidth of \(\cg'_n\) 
must be \(\geq n+1\) yet has a thinned jointree of width \(4\) (constructed by 
the algorithm in~\cite{DarwicheECAI20b}) so its causal treewidth is \(\leq 4\).
Recall that for this family of causal graphs \(\cg'_n\), the causal effect of \(X_1\) on \(Z_{nn}\) has the only back-door
\(X_2, \ldots, X_n\) so it has a back-door formula with a sum that is exponential in~\(n\) and
 a circuit of size \(O(n^2)\).

\section{Conclusion}
\label{sec:conclusion}

We discussed recent techniques that can exploit causal mechanisms computationally
without having to know their identities, which is the classical setup in causal inference. We also showed how
one can use these techniques to compile non-parametric causal graphs into circuits that can be used to estimate parameters from
data and to perform cross-layer causal inference in time linear in the circuit size. 
Our aim was to provide an intuitive exposure to these techniques to a causality audience who may not be
as familiar with them, with the hope that this may lead to a synthesis on how tractable arithmetic circuits
can aid causal inference in reaching higher levels of scalability and versatility. 

\textbf{Acknowledgements\ \ }
I wish to thank Elias Bareinboim, Yizuo Chen, Scott Mueller, Judea Pearl and Jin Tian for useful discussions and feedback.
This work has been partially supported by NSF grant \#ISS-1910317 and ONR grant \#N00014-18-1-2561.

\bibliographystyle{plain}
%\bibliography{bib/adnan,bib/references,bib/references2,bib/refs,bib/refsnn}

\end{document}